\documentclass{article}
\usepackage{arxiv}

\usepackage[utf8]{inputenc} 
\usepackage[T1]{fontenc}    
\usepackage{hyperref}       
\usepackage{url}            
\usepackage{booktabs}       
\usepackage{amsfonts}       
\usepackage{nicefrac}       
\usepackage{microtype}      
\usepackage{lipsum}
\usepackage{graphicx}

\usepackage{graphicx} 
\usepackage{amsthm}
\usepackage{amsmath}
\usepackage{xcolor}
\usepackage{setspace} 
\usepackage{hyperref}
\hypersetup{
  colorlinks   = true, 
  urlcolor     = red, 
  linkcolor    = blue, 
  citecolor   = blue 
}
\usepackage{amssymb}
\usepackage{algorithm}
\usepackage{algorithmic}
\usepackage{tikz}
\usepackage{subfig}

\newcommand{\aModel}[0]{M}
\newcommand{\aFeature}[0]{x}

\newcommand{\aFeatureSet}[0]{X}
\newcommand{\suffReasons}[0]{SR}
\newcommand{\anEntity}[0]{e}
\newcommand{\aSufReason}[0]{S}
\newcommand{\entities}[1]{\texttt{ent(}#1\texttt{)}}
\newcommand{\aDecisionTree}[0]{T}
\newcommand{\dnf}[1]{\texttt{cnf}(#1)}
\newcommand{\aDomain}[0]{\mathcal{D}}
\newcommand{\maxvalue}[0]{M}
\newcommand{\minvalue}[0]{m}

\newcommand{\defProblem}[3]{%
\begin{center}
\fbox{\begin{minipage}{33em}
    \textsc{Problem}: #1
    
    \textsc{Input}: #2
    
    \textsc{Output}: #3
\end{minipage}}
\end{center}
}

\newcommand{\pNecessaryFeatures}[1]{\textsc{Necessary}_{#1}}
\newcommand{\pRelevantFeatures}[1]{\textsc{Relevant}_{#1}}
\newcommand{\pUsefulFeatures}[1]{\textsc{Useful}_{#1}}
\newcommand{\pEquiv}[1]{\textsc{Equiv}_{#1}}
\newcommand{\aModelClass}[0]{\mathcal{M}}
\newcommand{\decisionTreeClass}[0]{\textsc{DT}}
\newcommand{\booleanDecisionTreeClass}[0]{\decisionTreeClass{}_{B}}
\newcommand{\fbddClass}[0]{\textsc{FBDD}}
\newcommand{\obddClass}[0]{\textsc{OBDD}}
\newcommand{\ddnnfClass}[0]{\textsc{d-DNNF}}
\newcommand{\dnfClass}[0]{\textsc{DNF}}

\newcommand{\NPcomplete}[0]{\textsc{NP-complete}}
\newcommand{\NP}[0]{\textsc{NP}}
\newcommand{\NPhard}[0]{\textsc{NP-hard}}
\newcommand{\sharpPcomplete}[0]{\textsc{\#P-complete}}
\newcommand{\coNP}[0]{\textsc{coNP}}

\newcommand{\val}[0]{\texttt{value}}
\newcommand{\values}[0]{\texttt{values}}
\newcommand{\feature}[0]{\texttt{feat}}
\newcommand{\consistWith}[2]{\texttt{cw}(#1, #2)}

\newtheorem{example}{Example}
\newtheorem{theorem}{Theorem}
\theoremstyle{plain} 
\newtheorem{proposition}[theorem]{Proposition}
\newtheorem{lemma}[theorem]{Lemma}
\newtheorem{corollary}[theorem]{Corollary}

\newtheorem{definition}[theorem]{Definition}

\begin{document}

\title{Feature Relevancy, Necessity and Usefulness: Complexity and Algorithms}

\author{Tomás Capdevielle \\ tomas.capdevielle@gmail.com\\ Departamento de Computación, UBA, Argentina \And
Santiago Cifuentes\thanks{Corresponding author} \\ scifuentes@dc.uba.ar\\
ICC Conicet, UBA, Argentina}

\maketitle

\begin{abstract}
    Given a classification model and a prediction for some input, there are heuristic strategies for ranking features according to their importance in regard to the prediction. One common approach to this task is rooted in propositional logic and the notion of \textit{sufficient reason}. Through this concept, the categories of relevant and necessary features were proposed in order to identify the crucial aspects of the input. This paper improves the existing techniques and algorithms for deciding which are the relevant and/or necessary features, showing in particular that necessity can be detected efficiently in complex models such as neural networks. We also generalize the notion of relevancy and study associated problems. Moreover, we present a new global notion (i.e. that intends to explain whether a feature is important for the behavior of the model in general, not depending on a particular input) of \textit{usefulness} and prove that it is related to relevancy and necessity. Furthermore, we develop efficient algorithms for detecting it in decision trees and other more complex models, and experiment on three datasets to analyze its practical utility.
\end{abstract}

\section{Introduction}

Artificial Intelligence (AI) has rapidly become an integral part of our daily lives, with applications of AI and machine learning models becoming increasingly popular \cite{brynjolfsson2014second}. Thanks to the greater affordability of computing resources \cite{mccallum2023memory} and the availability of vast public data repositories \cite{taylor2024amount}, these models can now be trained more extensively, achieving levels of performance that were once unimaginable \cite{brynjolfsson2017can,lu2019artificial,ng2016artificial,shao2022tracing}.

On the one hand, these advancements in model performance have led to impressive developments in generative AI and more precise results in classification models. On the other hand, as both the power and popularity of AI models have grown, several concerns regarding their applications have also arisen due to multiple reasons including machine bias, catastrophic failures involving AI models, and ethics \cite{angwin2022machine,gunnin2019xai,ribeiro2016should,tjoa2020survey}.

As the complexity and depth of modern AI systems increases, the ability of humans to \textit{understand} and \textit{interpret} their behavior diminishes \cite{narodytska2018verifying,ruan2018reachability}. Thus, the emergent field of XAI (eXplainable AI) aims to develop techniques and heuristics to evaluate and explain the decisions and outputs of the models. Such tools are crucial in applications of a sensitive nature, such as AI-assisted medical diagnosis and credit scoring. Explaining AI-driven decisions is a key component in advancing toward the broad goal of an ethical use of AI, according to recent initiatives such as in \cite{goodman2017european,european2018coordinated,madiega2019eu}. Moreover, the use of XAI tools extends beyond enhancing the security of AI application: by providing explanations for the results obtained by these models, XAI enables a deeper understanding of their behavior, which can aid in detecting problems and identifying potential areas for improvement.

One way of providing insight into a prediction for a particular input consists in ranking features according to their importance. Most common \textit{feature ranking} methods are grounded in game theory, such as the Shapley values \cite{shapley1953value}, and therefore possess theoretical properties desirable for the explainability task. However, computing the exact Shapley values is usually computationally expensive \cite{marzouk2025computational,arenas2023complexity,van2022tractability}, which is why approximations, such as the SHAP score \cite{lundberg2017unified} are often employed in practice. These approximations might be practical, but they are less accurate \cite{fryer2021shapley} and do not maintain all the desirable theoretical properties. 

Both the Shapley values and its approximations have been pointed out for assigning non-zero scores to features that are irrelevant from a logical perspective, and similarly assign zero scores to relevant ones \cite{huang2023inadequacy}. \textit{Relevancy}, in the logical context, is defined upon the notion of sufficient reason or abductive explanation \cite{marques2023logic}. Intuitively, a set of features $S$ is a sufficient reason for some prediction if said features determine the prediction by themselves: that is, changing the other features' values would not change the model result, as long as the values from the features from $S$ remain unchanged. Then, relevant features are defined as those features present in some sufficient reason, while necessary ones are defined as those belonging to all of them \cite[]{audemard2021explanatory}.

Detecting relevancy is hard in general, and thus the problem is usually considered in the context of simple models such as decision trees or restricted binary decision diagrams \cite[]{audemard2020tractable, huang2023feature}. While one might argue that these families of classifiers are too simplistic to be useful in practical environments, we observe that: first, they provide a valuable starting point for studying the feasibility of logic-based explainability in a world increasingly dominated by black-box systems; and second, it is possible to \textit{compile} complex systems into these simpler models in an \textit{off-line} manner to then perform fast \textit{online} explainability queries \cite{darwiche2004new,marquis2015compile,de2020lower,huang2021efficient, bertossi2023compiling}, 

\paragraph{Our contributions} In this paper we extend previous results on the complexity of detecting feature relevancy and necessity, and moreover provide efficient algorithms to solve these problems for families of models where the complexity was unknown. In particular, we show that relevancy can be detected in decision trees with numerical features, extending the result for general trees with categorical features \cite{huang2021efficient}. For the case of necessity, we extend the tractability frontier showing that it can be detected in any binary decision model, and even provide linear time algorithms to compute all necessary features for decision trees and \fbddClass{}s. We also extend the notion relevancy in two ways, show that they both are \textsc{NP-complete} to compute for decision trees, and provide efficient algorithms for restricted cases.

Finally, we propose a global notion to decide which features are important for the model in general, which we call \textit{usefulness}. We show that it is related to the notions of relevancy and necessity, and that the complexity of detecting it is related to the problem of deciding whether two models are equivalent. We also define a scoring notion based on it, and provide a general algorithm for computing it. In particular, the proposed procedure runs in quadratic time for decision trees. We compute this score for three different datasets and show that the ranking it induces is consistent with other scoring schemes.

\paragraph{Related work} The problem of computing relevant and necessary features was introduced in \cite{audemard2021explanatory} and solved for the case of binary decision trees with boolean features (i.e. decision trees classifying each input into one of two classes, where each feature of the input can be either 0 or 1). In \cite{huang2021efficiently} these results were extended for decision trees with generic categorical features and an arbitrary number of classes. More recently, in \cite{huang2023feature} hardness for detecting relevancy was shown for \fbddClass{}s, and polynomial time algorithms were developed to compute the set of necessary features for families as expressive as \ddnnfClass{} circuits. In \cite{darwiche2022computation} a decision tree model admitting both categorical and numerical features similar to ours was considered, but they did not study the problem of relevancy and necessity.

Other problems related to computing sufficient reasons were considered in the literature. In particular, \cite{audemard2023computing} studied these problems for boosted trees, and \cite{carbonnel2023tractable} for multivariate decision trees. In \cite{marques2023logic} the complexity of enumerating prime implicants of \ddnnfClass{}s was studied. \cite{izza2020explaining} performed a more experimental study on computing these explanations for decision trees, comparing it with the ``direct reason'' given by the classification path. Both \cite{audemard2023computing} and \cite{huang2021efficiently} consider the problem of computing contrastive explanations.

We observe that \cite{bertossi2020causality} proposed a scoring schema based on the notion of \textit{counterfactual cause} \cite{halpern2005causes} which is somewhat related to our notion of usefulness. As we will see, the usefulness of a feature $x$ is related to the number of entities that have a counterfactual cause on $x$.

\paragraph{Organization} In Section~\ref{sec:definitions} we present all definitions that we will use in the rest of the work, and in particular we present our decision tree model admitting both categorical and numerical features. In Section~\ref{sec:auxiliary_results} we present some auxiliary results related to hitting sets that we need for Section~\ref{sec:main_results}, where we show our main results and algorithms. In Section~\ref{sec:experiments} we experiment with three datasets to analyze the practical utility of our scoring proposal. Finally, in Section~\ref{sec:conclusions} we state some conclusions and propose future lines of work.

\section{Definitions}\label{sec:definitions}

When considering decision trees, we will work with categorical and numerical features, in a similar manner to \cite{darwiche2022computation}, but treating numerical features as first-class citizens: some of our results can be seen as generalizations of the well-known hitting set dualization results \cite{liffiton2008algorithms,reiter1987theory} already considered in \cite{huang2021efficiently,darwiche2022computation} to handle general categorical and numerical features. In addition, all the models that we consider will classify each entity into one of finite categories.

Let $\aFeatureSet$ be some finite set of elements, that we will consider as a set of features, and that we can partition as $\aFeatureSet = \aFeatureSet_C \cup \aFeatureSet_N$, where $\aFeatureSet_C$ and $\aFeatureSet_N$   denote the set of categorical and numerical features, respectively. For each categorical feature $x \in \aFeatureSet_C$ we consider given some domain $\aDomain_x$ such that $|\aDomain_x| < \infty$. We say that a categorical feature is binary if $|\aDomain_x| = 2$. Meanwhile, for each numerical feature $x \in \aFeatureSet_N$ we consider given some domain $\aDomain_x = [\minvalue_x,\maxvalue_x]$, where $\minvalue{}_x \in \mathbb{R} \cup \{-\infty\}$ denotes the smallest value that the feature can take, while $\maxvalue{}_x \in \mathbb{R} \cup \{\infty\}$ denotes the biggest one. We define the set of entities over $\aFeatureSet$ as $\entities{\aFeatureSet} = \{f: \aFeatureSet \to \bigcup_{x \in \aFeatureSet} \aDomain_x : f(x) \in \aDomain_x \, \forall x \in \aFeatureSet\}$. Given an entity $\anEntity \in \entities{\aFeatureSet}$ and some feature $\aFeature \in \aFeatureSet$ the expression $\anEntity(\aFeature)$ indicates the value that feature $x$ has for entity $e$. We say that a model is binary if it only uses binary categorical features.

A $k$-class model $\aModel$ over $\aFeatureSet$ is a mapping from \entities{$\aFeatureSet$} to $\{0,1,\ldots,k-1\}$\footnote{Whenever the context makes it clear, we will omit specifying the feature set over which the model or the entities are defined}. Given an entity $\anEntity \in \entities{X}$ the value $\aModel(\anEntity)$ indicates the class to which $\anEntity$ belongs according to $\aModel$. Given two models $M_1$ and $M_2$ we use $M_1 \equiv M_2$ to indicate that $M_1(e) = M_2(e)$ for all entities $\anEntity \in \entities{\aFeatureSet}$. We refer to $2$-class models as \textit{Boolean} classifiers, and we say that a boolean model $M$ \textit{accepts} an entity $e$ if $M(e) = 1$ and \textit{rejects} it otherwise.

Given some subset $S \subseteq \aFeatureSet$ of features and some entity $e \in \entities{\aFeatureSet}$ we define the set of features consistent with $e$ on $S$ as $\consistWith{e}{S} = \{e' \in \entities{\aFeatureSet} : e'(x) = e(x) \, \forall x \in S\}$. Given some entity $\anEntity$ we denote as $\anEntity_{x = b}$ with $x \in \aFeatureSet$ and $b \in \aDomain_{x}$ the unique entity satisfying $\anEntity(y) = \anEntity_{x=b}(y)$ for all $y \in \aFeatureSet \setminus \{\aFeature\}$, and $\anEntity_{x = b}(x) = b$. In a similar fashion, given a model $M$ over $\aFeatureSet$, a feature $x \in \aFeatureSet$ and $b \in \aDomain_x$ we define $M_{x=b}$ as the model satisfying $M_{x=b}(e) = M(e_{x=b})$.

A literal $l$ over features $X$ is an expression of one of the following forms:

\begin{enumerate}
    \item If $x \in \aFeatureSet_C$, then $l = x \, op \, D$ where $op \in \{\in, \notin\}$ and $D \subseteq \aDomain_x$.

    \item If $x \in \aFeatureSet_N$, then $l = x \, op \, b$ where $op \in \{\leq, >\}$ and $b \in \aDomain_x$.
\end{enumerate}

We denote the negation of a literal $l$ as $\lnot l$ and define it in the usual manner by flipping the operator involved in the literal (i.e. $\in$ changes to $\notin$, $\leq$ to $>$, etc). 

A \textit{term} is a conjunction of literals, and a \textit{clause} is a disjunction of literals. The size of a term or a clause is defined as its number of literals. A \textsc{DNF} formula is a disjunction of terms, and a \textsc{CNF} formula is a conjunction of clauses. The size of a \textsc{DNF}
or \textsc{CNF} formula is defined as the sum of the sizes of its terms or clauses, respectively. Any term $l = x \, op \, v$ can be understood as a boolean classifier: given any entity $e \in \entities{\aFeatureSet}$ we consider that $l(e) = 1 \iff e(x) \, op \, v$. This can be naturally extended to terms and clauses, and furthermore to \textsc{DNF} and \textsc{CNF} formulas. Whenever the feature $x$ is binary, we will write $x$ to denote the literal $x \in \{1\}$ and $\overline{x}$ to refer to $x \in \{0\}$.

\begin{example}\label{example:CNF}
    The formula 
    \begin{align*}
        \varphi(x_1,x_2, x_3, x_4, x_5) = (x_1 \vee \overline{x_2} \vee x_5) \wedge (x_2 \vee x_3 \vee x_4) \wedge (\overline{x_2} \vee x_4 \vee \overline{x_5}) \wedge (\overline{x_1} \vee \overline{x_2} \vee x_5)
    \end{align*}
    is a \textsc{CNF} formula, which can be understood as a boolean model with binary features. Given the entity $e= \{x_1: 0, x_2 : 0, x_3: 1, x_4: 1, x_5: 0\}$, it holds that $\varphi(e) = 1$.
\end{example}

We say that a boolean class of models $\aModelClass$ is closed under

\begin{itemize}
    \item \textit{Conditioning}, if whenever $M \in \aModelClass$, it is the case that $M_{x=b} \in \aModelClass$ for every $x \in \aFeatureSet$ and $b \in \aDomain_x$, and $M_{x=b}$ can be computed from $M$, $x$ and $b$ in polynomial time.

    \item \textit{Disjoint disjunction}, if whenever $M_1,M_2 \in \aModelClass$, it is the case that $M = (M_1 \wedge x) \vee (M_2 \wedge \overline{x}) \in \aModelClass$, where $x$ is a fresh binary categorical feature not used by neither $M_1$ or $M_2$, and $M$ can be computed efficiently given $M_1$ and $M_2$.

    \item \textit{Negation}, if whenever $M \in \aModelClass$ it holds that $\lnot M \equiv 1 - M\in \aModelClass$ and $\lnot \aModel$ can be computed efficiently given $\aModel$.

    \item \textit{Conjunction}, if whenever $M_1,M_2 \in \aModelClass$ it holds that $M=M_1 \wedge M_2 \in \aModelClass$ and $M$ can be computed efficiently given $M_1$ and $M_2$.
\end{itemize}

In this work, we will be interested in \textit{sufficient reasons}:

\begin{definition}[Sufficient Reason \cite{marques2023logic}]\label{def-ax}
    Given a model $\aModel$ and an entity $\anEntity$ we define the set of reasons for the prediction $M(e)$ as

    \begin{align*}
        R(M, e) = \left\{ 
        S \subseteq \aFeatureSet : M(e') = M(e) \, \forall e' \in \consistWith{e}{S} \right\}
    \end{align*}

    A sufficient reason is a reason that is minimal with respect to the property of being a reason. Thus, we define the set of sufficient reasons as

    \begin{align*}
        \suffReasons{}(M, e) = \{S \subseteq X : S \in R(M, e), \forall S' \subset S, S' \notin R(M, e)\}
    \end{align*}
\end{definition}

Sufficient reasons are minimal subsets of features such that, if their values are preserved as stated by $e$, then the behavior of the model $M$ does not change, no matter the modifications that other features can suffer.

\begin{example}\label{example:suf_reasom}

    Consider the model $\varphi$ and the entity $e$ from Example~\ref{example:CNF}. Then, it can be checked that both $\{x_2, x_3\}$ and $\{x_2, x_4\}$ are sufficient reasons. Moreover, it can be shown that they are the only sufficient reasons, and thus  $\suffReasons{}(\varphi, e) = \{\{x_2, x_3\}, \{x_2, x_4\}\}$.
    
\end{example}

A model may have many sufficient reasons (potentially an exponential number relative to its size\cite{audemard2021explanatory}), and thus any individual reason might not be useful for the explanatory task. Nevertheless, we can distinguish important features by considering those that belong to some or all sufficient reasons. These correspond to the notions of relevant and necessary features, respectively.

\begin{definition}[Relevant feature]\label{def-rf}
Given a model $\aModel$, a feature $\aFeature \in \aFeatureSet$ and an entity $\anEntity \in \entities{\aFeatureSet}$ over $\aModel$ we say that $\aFeature$ is relevant for the prediction $\aModel(\anEntity)$ if there is a sufficient reason $\aSufReason \in \suffReasons{}(\aModel, \anEntity)$ such that $\aFeature \in \aSufReason$.
\end{definition}

\begin{definition}[Necessary feature]\label{def-nf}

Given a model $\aModel$, a feature $\aFeature \in \aFeatureSet$ and an entity $\anEntity \in \entities{\aFeatureSet}$ over $\aModel$ we say that $\aFeature$ is necessary for the prediction $\aModel(\anEntity)$ if for all sufficient reasons $\aSufReason \in \suffReasons{}(M, \anEntity)$ it is the case that $\aFeature \in \aSufReason$.

\end{definition}

The two previous notions are local, i.e. relevant and necessary features can be used to understand the behavior of model $M$ regarding the prediction for some particular entity $\anEntity$. We introduce a more global notion intended to capture those features that the model uses to distinguish at least two entities.

\begin{definition}[Useful feature]
    Given a model $\aModel$ and a feature $\aFeature \in \aFeatureSet$ we say that $\aFeature$ is useful if there is some $e \in \entities{\aFeatureSet}$ and $b \in \aDomain_x$ such that $M(\anEntity) \neq \aModel(\anEntity_{x = b})$.
\end{definition}

Intuitively, a feature that is not useful is not required by the model because there is not a single entity whose classification depends on the value of that feature. This intuition is formally correct, since if $x$ is not useful for model $M$ then $M \equiv M_{x=b}$  for any $b \in \aDomain_x$. We will show that the notion of usefulness is related to relevancy and necessity (see Proposition~\ref{prop:useful_features_are_necessary_for_someone_and_also_relevant}). We will also propose counting the number of entities for which $x$ is useful to assign a value (a \textit{score}) to the feature $x$ in order to rank feature importance. 

\begin{example}\label{example:rel_nec_use}
    Consider again the model $\varphi$ and the entity $e$ from Example~\ref{example:CNF}. Observing the sufficient reasons already computed in Example~\ref{example:suf_reasom}, it holds that $x_2$ is a necessary feature for $\varphi(e)$ while $x_3$ and $x_4$ are relevant.

    Furthermore, $x_1$ is not useful for $\varphi$ (note that the first clause can be simplified with the last one), while the other features are all useful.
\end{example}

We will now present all the families of models that we will be using in our results and algorithms, which have been used as compilation languages for more complex models \cite{darwiche2002knowledge} and are common in the XAI literature \cite{arenas2023complexity,barcelo2020model,huang2022tractable}. The simplest family is the one corresponding to decision trees. Our definition is somewhat non-standard because we allow for numerical and categorical features together. 

\begin{definition}[Decision tree]

A $k$-class decision tree over features $\aFeatureSet$ is a binary tree $\aDecisionTree$ (i.e. each node except for the leaves has two children), each of whose internal nodes is labeled with an element from $\{(x, D) : x \in \aFeatureSet_C, D \subseteq \aDomain_x\}$ (in which case we call it a categorical node) or rather from $\{(x, b) : x\in \aFeatureSet_N, b \in \aDomain_{x}\}$ (in which case it is a numerical node); and leaves are labeled with a class in the range $\{0,\ldots, k-1\}$. We assume the set of children of each internal node $v$ is represented by an ordered list $\{w_1, w_2\}$, and therefore we can refer to $w_1$ as the ``left'' child and to $w_2$ as the ``right'' child. The size of a decision tree $\aDecisionTree$ is defined as its number of nodes, and noted as $|\aDecisionTree|$. 

Given any categorical node $v$ we denote as $\feature{}(v) \in \aFeatureSet_C$ and $\values{}(v) \subseteq \aDomain_{\feature{}(v)}$ the feature and set of values corresponding to its label, respectively. We associate a literal $l(v)$ to each categorical node $v$ defined as $l(v) \equiv \feature{}(v) \in \values{}(v)$. Given any numerical node $v$ in $\aDecisionTree$, we denote as $\feature{}(v) \in \aFeatureSet_N$ and $\val{}(v) \in \aDomain_{\feature(v)}$ the feature and value of its label, respectively. We associate a literal $l(v)$ to each numerical node $v$ defined as $l(v) = \feature{}(v) \leq \val{}(v)$

Given $\anEntity \in \entities{\aFeatureSet}$ we denote the prediction of $\aDecisionTree$ for $\anEntity$ as $\aDecisionTree(\anEntity)$, and define it as the label of the leaf reachable from the root using the following strategy: when on a internal node $v$ move to the left child if $l(v)(e) = 1$, and move to the right child otherwise.
    
\end{definition}

See Figure~\ref{fig:decision-tree-example} for an example of a decision tree. Without loss of generality, we assume that there are no categorical nodes with labels of the form $(x, \emptyset)$ or $(x, \aDomain_x)$, and similarly that there are no numerical nodes with labels of the form $(x, M_x)$. If $v$ is a node from $\aDecisionTree$, we refer by $T_v$ to the decision (sub-)tree rooted at $v$. We denote the class of decision trees as $\decisionTreeClass$, while we denote the class of boolean binary decision trees as $\booleanDecisionTreeClass$.

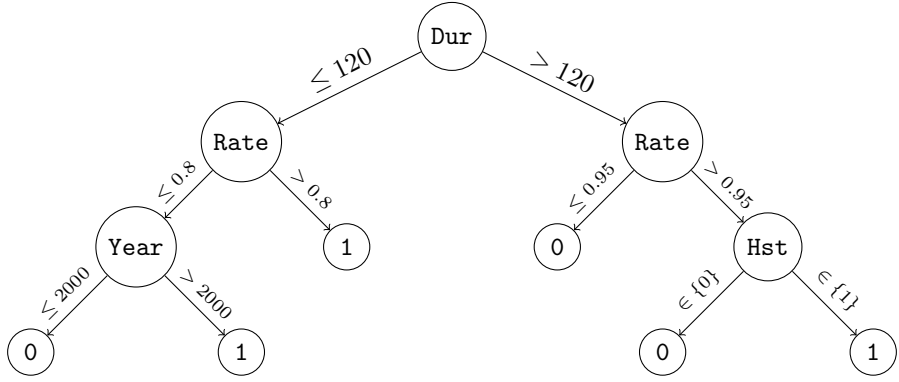
\begin{figure}
    \centering
    \begin{tikzpicture}[scale=0.7]
        \node (Dur) at (0,0) [draw, circle] {\texttt{Dur}};
        \node (Rate1) at (-4, -2) [draw, circle] {\texttt{Rate}};
        \node (Rate2) at (4, -2) [draw, circle] {\texttt{Rate}};
        \node (accept1) at (-2, -4) [draw, circle] {\texttt{1}};
        \node (reject1) at (2, -4) [draw, circle] {\texttt{0}};
        \node (year) at (-6, -4) [draw, circle] {\texttt{Year}};
        \node (hist) at (6, -4) [draw, circle] {\texttt{Hst}};
        \node (accept2) at (8, -6) [draw, circle] {\texttt{1}};
        \node (reject2) at (4, -6) [draw, circle] {\texttt{0}};
        \node (reject3) at (-8, -6) [draw, circle] {\texttt{0}};
        \node (accept3) at (-4, -6) [draw, circle] {\texttt{1}};
    
        \draw[->] (Dur) -- (Rate1) node[pos=0.5, sloped, above] {$\leq 120$};
        \draw[->] (Dur) -- (Rate2) node[pos=0.5, sloped, above] {$> 120$};
        \draw[->] (Rate1) -- (year) node[pos=0.5, sloped, above] {\scriptsize $\leq 0.8$};
        \draw[->] (Rate1) -- (accept1) node[pos=0.5, sloped, above] {\scriptsize $> 0.8$};
        \draw[->] (year) -- (reject3) node[pos=0.5, sloped, above] {\scriptsize $\leq 2000$};
        \draw[->] (year) -- (accept3) node[pos=0.5, sloped, above] {\scriptsize $> 2000$};
        \draw[->] (Rate2) -- (reject1) node[pos=0.5, sloped, above] {\scriptsize $\leq 0.95$};
        \draw[->] (Rate2) -- (hist) node[pos=0.5, sloped, above] {\scriptsize $> 0.95$};
        \draw[->] (hist) -- (reject2) node[pos=0.5, sloped, above] {\scriptsize $\in \{0\}$};
        \draw[->] (hist) -- (accept2) node[pos=0.5, sloped, above] {\scriptsize $\in \{1\}$};
    \end{tikzpicture}
    \caption{Example of a decision tree representing a recommendation system for a film database. The set of features is \texttt{Dur} (duration), \texttt{Rate}, \texttt{Year} and \texttt{Hst} (whether the film is of the historical genre). The domains are $\aDomain_{\texttt{Dur}} = [0, \infty]$, $\aDomain_{\texttt{Rate}} = [0,1]$, $\aDomain_{\texttt{Year}} = [1888, \infty]$ and $\aDomain_{\texttt{Hst}} = \{0, 1\}$. Note that the first three features are numerical, while the last one is categorical. The tree classifies entity $e = \{\texttt{Dur}:90, \texttt{Rate}:0.85, \texttt{Year}:2005, \texttt{Hst}: 0\}$ to $1$. Moreover, the sufficient reasons for the result are $\{\texttt{Dur}, \texttt{Rate}\}$ and $\{\texttt{Dur}, \texttt{Year}\}$, and therefore $\texttt{Dur}$ is a necessary feature while $\texttt{Rate}$ and $\texttt{Year}$ are both relevant.}
    \label{fig:decision-tree-example}
\end{figure}

One of our algorithms will involve \fbddClass{}s, and thus we lay out some definitions.

\begin{definition}[\fbddClass{} (\textit{Free Binary Decision Diagram})]
    A \textit{decision diagram} is a boolean binary model given by a directed acyclic graph $D$ with a unique node with an in-degree of 0 identified as the root, and such that each node has an out-degree equal to either 2 (the internal nodes) or 0 (the leaves). Each internal node $v$ has its set of reachable nodes ordered as $N(v) = \{w_1, w_2\}$ and thus we can refer to the left and right child of $v$ as $w_1 = \texttt{left}(v)$ and $w_2 = \texttt{right}(v)$. Each internal node $v$ has an associated feature $\feature{}(v)$, and each leaf has an associated label $0$ or $1$. 

    Given an entity $e \in \entities{X}$ (where all features are binary) we denote the prediction of $D$ for $e$ as $D(e)$, and define it as the label of the leaf reachable from the root using the following strategy: when on an internal node $v$ move to $\texttt{left}(v)$ if $e(\feature{}(v)) = 0$, and move to $\texttt{right}(v)$ otherwise.

    A decision diagram $D$ is an \fbddClass{} if every directed path $P = v_0, \ldots, v_k$ in $D$ satisfies that $\feature{}(v_i) \neq \feature{}(v_j)$ for $0 \leq i < j \leq k$ (i.e. features are not repeated in any path). This is usually referred to as the \textit{read-once property}.

\end{definition}

See Figure~\ref{fig:fbdd-example} for an example of an \fbddClass{}. If $D$ is an \fbddClass{} and $v$ is a node from $D$, we use $D_v$ to denote the \fbddClass{} obtained by taking $v$ as the root of $D$.

\begin{figure}
    \centering
    
    \begin{tikzpicture}
        \node (x2) at (0,0) [draw, circle] {$x_2$};
        \node (x3) at (-2,-2) [draw, circle] {$x_3$};
        \node (x5) at (2,-2) [draw, circle] {$x_5$};
        \node (x4) at (-2,-4) [draw, circle] {$x_4$};
        \node (0) at (-2,-6) [draw, circle] {0};
        \node (1) at (2,-6) [draw, circle] {1};
    
        \draw[->] (x2) -- (x3) node[pos=0.5, sloped, above] {$0$};
        \draw[->] (x2) -- (x5) node[pos=0.5, sloped, above] {$1$};
        \draw[->] (x3) -- (x4) node[pos=0.5, above left] {$0$};
        \draw[->] (x3) -- (1) node[pos=0.5, sloped, above] {$1$};
        \draw[->] (x4) -- (0) node[pos=0.5, above left] {$0$};
        \draw[->] (x4) -- (1) node[pos=0.5, sloped, above] {$1$};
        \draw[->] (x5) -- (x4) node[pos=0.5, sloped, above] {$0$};
        \draw[->] (x5) -- (1) node[pos=0.5, above left] {$1$};
    \end{tikzpicture}
        \caption{An \fbddClass{} representing the CNF formula from Example~\ref{example:CNF}. Note that each directed path does not contain two nodes with the same feature.}
        \label{fig:fbdd-example}
        
\end{figure}
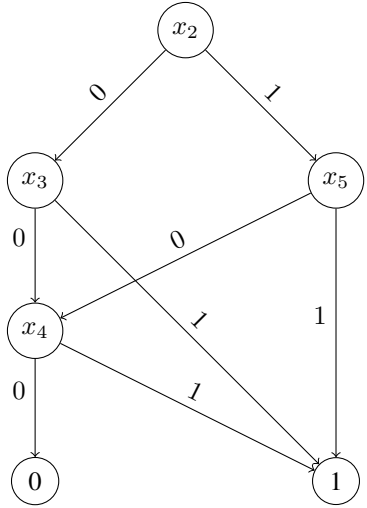

The other models we mention are standard and their particular details are not needed for the development of this work. These encompass \obddClass{}s (Ordered Binary Decision Diagrams), \ddnnfClass{}s and general \dnfClass{}s \cite{darwiche2002knowledge}. All of them are understood as particular cases of boolean binary models. Moreover, all model classes mentioned here are closed under conditioning and disjoint disjunction, all binary diagrams are closed under negation, and trees are also closed under conjunction. 

We recall that any decision tree can be represented as a CNF formula of small size.

\begin{definition}[CNF of a boolean decision tree]
    Let $T$ be a boolean decision tree. We define its set of paths as $P_{\aDecisionTree} = \{v_0, v_1 ,\ldots, v_k :$ $v_0$ is the root of $T$, $v_k$ is a leaf and $v_{i+1}$ is a child of $v_{i}$ for $0 \leq i < k\}$. Given a path $v_0,\ldots,v_k = p \in P_{\aDecisionTree}$ we associate a term $F_p$ to $p$ as

    \begin{align*}
        F_p &= \bigwedge_{\substack{i=0 \\ v_{i+1} \text{ is left child of } v_i}}^k l(v_i) \,\,\wedge \bigwedge_{\substack{i=0 \\ v_{i+1} \text{ is right child of } v_i}}^k \lnot l(v_i) 
    \end{align*}

    and likewise a clause $C_p = \lnot F_p$ (where negation propagates following De Morgan laws).

    Let $P_{\aDecisionTree}^0 \subset P_{\aDecisionTree}$ be the subset of paths of $\aDecisionTree$ that end at a node with label $0$. Then, it holds that $\aDecisionTree \equiv \lnot \bigvee_{p \in P_{\aDecisionTree}^-} F_p = \bigwedge_{p \in P_{\aDecisionTree}^-} C_p$. We denote the boolean formula $\bigwedge_{p \in P_T^-} C_p$ as $\dnf{\aDecisionTree}$. It holds that $|\dnf{\aDecisionTree}| = O(|\aFeatureSet||\aDecisionTree|)$.
\end{definition}


    



Even though any decision tree is represented by a succinct (i.e. with at most a polynomial overhead in size) CNF formula, it is not the case that any CNF formula can be represented succinctly by a decision tree. Nevertheless, it is a well known fact that any model with a small set of positive entities can be represented by a decision tree:

\begin{lemma}[Decision trees represent sparse models]\label{lemma:trees_sparse_models}
   Let $E \subseteq \entities{\aFeatureSet}$ be a subset of entities over features $\aFeatureSet$, where all features from $X$ are categorical. Then, it is possible to build a boolean decision tree $\aDecisionTree$ such that $\aDecisionTree(\anEntity) = 1$ if and only if $\anEntity \in E$. This construction takes time $O(|E||\aFeatureSet|)$, and the resulting decision tree has size $O(|E||X|)$.
\end{lemma}

\begin{proof}
    We build the tree iteratively, assuming that  we start from the tree that rejects all entities, i.e. the one with only one node (the root) with label 0. Let $E = \{\anEntity_1, \ldots, \anEntity_k\}$, and assume that $T_i$, for some $0 \leq i < k$, is a tree that accepts all entities $e_1,\ldots,e_i$ and rejects any other one. Then, given $e_{i+1}$ we can follow the path defined in $T$ for $e_{i+1}$ and we will reach a leaf with label 0. We then extend this tree by appending to this leaf a model that accepts only $e_{i+1}$ (which can be straightforwardly built as a path of length $2|\aFeatureSet|$ because all features are categorical). The resulting tree $T_{i+1}$ only accepts entities $e_1,\ldots,e_{i+1}$, and its size is bounded by $|T_{i}| + 2|\aFeatureSet|$.
\end{proof}

Given a \textsc{CNF} formula $\varphi$ and an entity $\anEntity$ it is possible to simplify $\varphi$ considering only those atoms that coincide with the assignment dictated by $\anEntity$. This new \textsc{CNF} formula is useful for characterizing the set of sufficient reasons for $\anEntity$ with respect to $\varphi$.

\begin{definition}[CNF formula restricted to an entity]
    Let $\varphi = \bigwedge_{i=1}^m C_i$ be a CNF formula where each $C_i$ is a clause over $\aFeatureSet$. For any entity $\anEntity \in \entities{\aFeatureSet}$ and clause $C = \bigvee_{j=1}^m l_j$ we define:
    
    \begin{align*}
        C^e = \bigvee_{\substack{j=1 \\ l_j(e) = 1}}^k l_j
    \end{align*}

    Then, we define the model $M$ restricted to $\anEntity$ as $M_e = \bigwedge_{i=1}^m C_i^e$.
\end{definition}

 More precisely, the sufficient reasons for the prediction $\varphi(e)$ will exactly coincide with the set of minimal hitting sets of a certain hypergraph built from $\varphi_e$. This is well-known \cite{liffiton2008algorithms,reiter1987theory} and was used in \cite{audemard2021explanatory} to develop an algorithm for computing the relevant and necessary features for predictions of boolean binary decision trees. In \cite{huang2021efficiently} this idea was further extended for a more general class of classifiers, observing that $\dnf{T}_e$ is the disjunction of all contrastive explanations \cite{marques2023logic} for prediction $T(e)$.
 
 An hypergraph is a tuple $H=(V, E)$ where $V$ is a set of nodes and $E \subseteq \mathcal{P}(V)$ is a set of hyperedges. We define the size of an hypergraph as $|H| = |V| + \sum_{B \in E} |B|$, and the degree of a node $v \in V$ as $\deg{}(v) = |\{B \in E : v \in B\}|$. A hitting set of $H$ is a subset $S \subseteq V$ of its nodes such that for any $B \in E$ it is the case that $S \cap B \neq \emptyset$. A hitting set $S$ is minimal if for all $S' \subset S$ it holds that $S'$ is not a hitting set.

\section{Auxiliary results}\label{sec:auxiliary_results}

We recall the result from \cite{liffiton2008algorithms}\footnote{Which refers to Theorem 4.5 from \cite{birnbaum2003consistent}} and state it in our terms:

\begin{lemma}[Sufficient reasons as Hitting Sets]\label{lemma:suf_reasons_as_hitting_set}

    Let $M \equiv \bigwedge_{i=1}^m C_i$ be a boolean model over $\aFeatureSet$ given as a CNF formula, where each $C_i$ is not tautologically true, and let $\anEntity \in \entities{\aFeatureSet}$ be an entity such that $\aModel(\anEntity) = 1$.
    
    For any clause $C$, we denote as $\texttt{var}(C)$ the set of features from $\aFeatureSet$ that appear in the clause $C$. We can define a hypergraph $H = (\aFeatureSet, E)$ using $\aFeatureSet$ as the set of nodes; and as hyperedges,

    \begin{align*}
        E = \{\texttt{var}(C_i^e) : 1 \leq i \leq m \}
    \end{align*}

    Then, it holds that $S \in \suffReasons{}(M, e)$ if and only if $S$ is a minimal hitting set of $H$. Also, $|H| = |\aFeatureSet| + \sum_{i=1}^m |C_i^e|$.
    
\end{lemma}

\begin{proof}
    Let $S \subseteq \aFeatureSet$. If $S \cap \texttt{var}(C_i^e) = \emptyset$ for some $i$, let $e_i$ be an entity such that $C_i^e(e_i) = 0$ (which exists because of our assumption that the clauses were not tautologically true), and consider the entity $e'$ defined as

    \begin{align*}
        e'(x) = \begin{cases}
            e(x) & x \in \aFeatureSet \setminus \texttt{var}(C_i^e) \\
            e_i(x) & \text{otherwise}
        \end{cases}
    \end{align*}

    Since $C_i^e(e') = 0$ it holds that $M(e') = 0$. Hence, $S\notin \suffReasons{}(M, e)$.

    On the other hand, let's see that if $S$ is a hitting set of $H$ then it must be a reason for $\anEntity$ with respect to model $M$. Let $e'$ be an entity that agrees with $e$ on every feature in $S$. Then, given any clause $C_i^e$ there is some feature $x \in S \cap C_i^e$ and thus there is some literal in $C_i^e$ that is satisfied by $e'$. This implies that $C_i^e(e') = 1$, and because this argument works for any $i$, it must be the case that $M(e') = 1$.
    
\end{proof}

Moreover, it holds that given any hypergraph $H$ one can build a boolean decision tree $\aDecisionTree$ and an entity $e$ such that the minimal hitting sets of $H$ are precisely the sufficient reasons for $e$ with respect to model $T$.

\begin{lemma}[From hitting sets to decision trees]\label{lemma:from_hitting_to_trees}

    Let $H = (V, E)$ be a hypergraph. Then, it is possible to build a boolean binary decision tree $T$ over binary features $V$ and an entity $e$ over $V$ such that the set of minimal hitting sets of $H$ is precisely the set of sufficient reasons $\suffReasons{}(T, e)$.

    Moreover, this construction can be done in $O(|V|^2)$ time and the resulting tree has size $O(|V|^2)$.
    
\end{lemma}

\begin{proof}

    Fix the entity $e$ as $e(v) = 1$ for all $v \in V$. For any $B \in E$ we define a clause $C_B^+ \equiv \bigvee_{v \in B} v \in \{1\}$ and likewise a clause $C_B^- \equiv \bigvee_{v \in V \setminus B} v \notin \{1\}$. Consider the model

    \begin{align*}
        M = \bigwedge_{B \in E} (C_B^+ \vee C_{B}^-) 
    \end{align*}

    Note that each clause $C_B^+ \vee C_{B}^-$ is satisfied by all entities but one. Therefore, $M$ rejects only $|E|$ entities. Then, $M$ can be represented by a decision tree using Lemma~\ref{lemma:trees_sparse_models}\footnote{More precisely, one can use Lemma~\ref{lemma:trees_sparse_models} to construct $\lnot M$, and then complement all the leaves to get a decision tree representing $M$.}. The constructed tree $T$ satisfies $\dnf{\aDecisionTree} = \bigwedge_{B \in E} (C_B^+ \vee C_B^-)$, and $\dnf{\aDecisionTree}_e = \bigwedge_{B \in E} (C_B^+ \vee C_B^{-})^e = \bigwedge_{B \in E} C_B^+$. By Lemma~\ref{lemma:suf_reasons_as_hitting_set} it holds that the sufficient reasons for $\anEntity$ with respect to $\aModel$ coincide with the minimal hitting sets of the hypergraph $H' = (V, \{\text{var}(C_B^+) : B \in E\})$, but since $\texttt{var}(C_B^+) = B$ it is the case that $H' = H$.
    
\end{proof}

The connection between the sufficient reasons of a decision tree $\aDecisionTree$ and the hitting sets of the hypergraph induced by $\aDecisionTree$ will be useful for proving both upper and lower complexity bounds. Regarding the latter, we provide two simple algorithms for computing minimal hitting sets of hypergraphs.

\begin{lemma}\label{lemma:structured_hitting_set}
    Given a hypergraph $H = (V,E)$ and a subset $W \subseteq V$ of its nodes it is possible to decide if there is a minimal hitting set containing $W$ in $O(|H|\prod_{w \in W} \deg{}(w))$.
\end{lemma}

\begin{proof}
   The algorithm follows from this simple observation: if $S$ is a minimal hitting set and $v \in S$, then there must be some $B \in E$ such that $S \cap B = \{v\}$. Note that otherwise $S \setminus \{v\}$ would also be a hitting set, contradicting the minimality of $S$. Thus, if $S$ is a hitting set such that $W \subseteq S$ it must be the case that for each $w \in W$ there is some $B_w \in E$ with $S \cap B_w = \{w\}$.

    The algorithm consists of exhaustive search considering all possibilities for the choice of $\{B_w\}_{w \in W} \subseteq E$, which are at most $\prod_{w \in W} \deg{}(w)$. Given any choice of $\{B_w\}_{w \in W}$ we need to decide if there is a hitting set included in $V \setminus \left(\bigcup_{w\in W} \left(B_w \setminus \{w\}\right)\right)$, which is the same as checking if $B \nsubseteq \bigcup_{w \in W} \left(B_w \setminus \{w\}\right)$ for all $B \in E$. If this holds, then there is a minimal hitting set of $H$ contained in $V \setminus \bigcup_{w \in W} B_w \setminus \{w\}$, and this minimal hitting set must contain each $w$, since otherwise the hyperedge $B_w$ would not be covered.

    Given a choice of $\{B_w\}_{w \in W}$ we can decide in $O(|H|)$ time whether $V \setminus \left(\bigcup_{w\in W}\left( B_w \setminus \{w\}\right) \right)$ is a hitting set of $H$. Thus, the final complexity is $O(|H|\prod_{w \in W} \deg{}(w))$.

\end{proof}

\begin{lemma}\label{lemma:construct_distinct_hitting_sets}
        Given a hypergraph $H = (V,E)$ and a node $v \in V$ it is possible to decide if there are $k$ different minimal hitting sets containing $v$ in $O(k |H| (\deg{}(v) |V|^{k-1} + |V|))$. 
\end{lemma}

\begin{proof}
    Assuming we have computed $k-1$ minimal hitting sets $S_1,\ldots, S_{k-1}$, we show how to compute another one, namely $S_k$. Using the same observation from Lemma~\ref{lemma:structured_hitting_set} we conclude that if $S$ is a minimal hitting set containing $v$, then there must be some hyperedge $B \in E$ such that $S \cap B = \{v\}$. Moreover, if we want a hitting set $S_k$ such that $S_i \neq S_k$ then there must be some node $s_i \in S_i$ such that $s_i \notin S_k$, and this holds for each $1 \leq i \leq k-1$.

    This reasoning implies an exhaustive search algorithm analogous to the one from Lemma~\ref{lemma:structured_hitting_set}. If there is another minimal hitting set $S_k$ that contains $v$ then there must be some $B \in E$ with $v \in B$ and nodes $s_1,\ldots,s_{k-1}$ with $s_i \in S_i$ such that $S_k \subseteq V \setminus \left(\left(B \setminus \{x\}\right) \cup \{s_1,\ldots,s_{k-1}\}\right)$. Therefore, for each choice of $B$ (of which there are $\deg{}(v)$) and of the nodes $s_1,\ldots,s_{k-1}$ (which are at most $|V|^{k-1}$) we can check whether $V \setminus \left(\left( B \setminus \{v\}\right) \cup \{s_1, \ldots, s_{k-1}\} \right)$ is a hitting set, and if it is then we can obtain $S_k$ by iteratively removing nodes until achieving minimality. This last procedure can be implemented in time $O(|H||V|)$, and thus we can find an $S_k$ if it exists in time $O(|H|\deg(v)|V|^{k-1} + |H||V|)$.

    To compute all the minimal hitting sets we need to invoke the previous algorithm $k$ times, and thus we can bound the whole complexity by $O(k|H|(\deg(v)|V|^{k-1} + |V|))$.
\end{proof}

The algorithm from Lemma~\ref{lemma:structured_hitting_set} is not polynomial with respect to the input size, and this situation cannot be improved in general:

\begin{proposition}\label{prop:cover_with_structure_np_complete}

    The problem of deciding, given a graph $G=(V, E)$ and a subset of nodes $W \subseteq V$, whether there is a minimal vertex cover\footnote{In the context of graphs a hitting set is a vertex cover.} $S$ of $G$ such that $W \subseteq S$ is \NPcomplete{}.
    
\end{proposition}

\begin{proof}
    The \NP{} membership is immediate. For the hardness, we show a reduction from \textsc{3-COLORING} to our problem. Nonetheless, for clarity we actually provide a reduction of \textsc{3-COLORING} to another problem we named \textsc{HITTING-SET-FP}, and then a reduction from \textsc{HITTING-SET-FP} to our problem. 
    
    \textsc{HITTING-SET-FP} is the problem of finding a hitting set of a hypergraph when some pairs of nodes are forbidden from being picked together. That is, the problem consists of deciding, given a hypergraph $H=(V, E)$ and some forbidden pairs $F \subseteq V \times V$, whether there is a hitting set $S$ of $H$ such that for every pair $(v_1,v_2) \in F$ it is the case that $\{v_1, v_2\} \nsubseteq S$. It is clear that this problem is in \NP{}.

    We now provide the reduction from \textsc{3-COLORING} to \textsc{HITTING-SET-FP}. Given a graph $G = (V_G, E_G)$ we define a hypergraph $H = (V_H, E_H)$ as

    \begin{align*}
        V_H &= \{v_i : v \in V_G, 1\leq i \leq 3\}\\
        E_H &= \{ \{v_1, v_2, v_3\} : v \in V_G\}\}
    \end{align*}

    with forbidden pairs $F = \{(v_i, w_i ) : vw \in E_G, 1\leq i \leq 3\} \cup \{(v_i,v_j) : v\in V_G, 1 \leq i < j \leq 3\}$. 
    
    Intuitively, each node of $H$ corresponds to an assignment of a node from $G$ to a color. The set of hyperedges enforces that one color must be chosen for each node, and the set of forbidden pairs enforces that nodes joined by an edge must use different colors and that each node must have a unique color. The correctness of the reduction is immediate.

    For the second reduction (from \textsc{HITTING-SET-FP} to our problem), consider a hypergraph $H = (V_H, E_H)$ with forbidden pairs $F$. We define a graph $G=(V_G, E_G)$ as

    \begin{align*}
        V_G &= V_H \cup E_H \\
        E_G &= \{vw : (v,w) \in F\} \cup \{(v, B) : B \in E_H, v \in B\}
    \end{align*}

    $G$ has all the nodes $H$ has, and one extra node for each hyperedge. Edges are put between the two nodes of every forbidden pair and between any node from $V_H$ and the hyperedges to which it belongs.

    We now show that $V_H$ admits a hitting set that does not contain any forbidden pair from $F$ if and only if $V_G$ admits a minimal hitting set containing $E_H$. This implies the correctness of the reduction by picking $W = E_H$.

    Let $S$ be a valid hitting set of $H$. Then each edge $v B \in E_G$ is covered by $V_G \setminus S$ because $B \in E \subseteq V_G \setminus S$. Similarly, for every $(v,w) \in F$ it is the case that either $v$ or $w$ does not belong to $S$ , and thus one of them belongs to $V_G \setminus S$, which implies that the edge $vw$ from $G$ is covered. Therefore, since $V_G \setminus S$ is a vertex cover there is some minimal vertex cover $C \subseteq V_G \setminus S$. We show that $B \in C$ for any $B \in E$: because $S$ is a hitting set, it holds that there is some $s_B \in S \cap B$, and thus if it were the case that $B \notin C$, then the edge $s_B B$ would be left uncovered. We conclude that there is a minimal vertex cover of $V_G$ containing $E_H$.

    For the other direction, let $C$ be a minimal vertex cover of $V_G$ such that $E_H \subseteq C$. Then, we claim that $V_H \setminus C$ is a valid hitting set of $H$. For each $B \in E$ it holds that each edge $vB$ is covered in $V_G$, for every $v \in B$. Since the covering is minimal, there must be some node $s_B \in B$ such that $s_B \notin C$, and thus $s_B \in V_H \setminus C$. This implies that $V_H \setminus C$ is a hitting set. To see that it is valid, consider any forbidden pair $(v, w)$: since $vw$ is an edge from $V_G$ it must be the case that either $v \in C$ or $w \in C$, and thus either $v \notin V_H \setminus C$ or $w \notin V_H \setminus C$. We conclude that $V_H \setminus C$ is a valid hitting set.
\end{proof} 

We presented Lemmas~\ref{lemma:structured_hitting_set} and ~\ref{lemma:construct_distinct_hitting_sets} merely to show that these problems, which are intractable in general (i.e. when either $|W|$ or $k$ is not fixed), may allow for efficient solutions in restricted cases. In case one needed to enumerate hitting sets efficiently these algorithms will not be optimal, and should rather resort to more efficient alternatives  \cite{gainer2017minimal}. 

\section{Main results}\label{sec:main_results}



\subsection{Relevant features}

Consider the following problem:

\defProblem{$\pRelevantFeatures{\aModelClass}$}
    {A model $M \in \aModelClass$ over $\aFeatureSet$, an entity $e \in \entities{\aFeatureSet}$ and a feature $\aFeature \in \aFeatureSet$.}
    {Is $x$ relevant for the prediction $M(e)$?}

In \cite{audemard2021explanatory} this problem was shown to be tractable for $\mathcal{M} = \booleanDecisionTreeClass$. \cite{huang2021efficiently} extended this result for a more general class of trees (in our terms, trees with only categorical features) and later \cite{huang2023feature} showed that it is \NPcomplete{} for \fbddClass{}s. We show that the algorithm from \cite{audemard2021explanatory} can be extended to our more general version of decision trees, without affecting its complexity\footnote{We note that the complexity stated in Proposition~\ref{teo:relevant_features_for_trees} actually does not agree with the one stated in \cite{audemard2021explanatory}[Proposition 6]. We believe the correct bound for their algorithm is $O(|\aFeatureSet||\aDecisionTree|^2)$.}:

\begin{theorem}\label{teo:relevant_features_for_trees}
    The problem $\pRelevantFeatures{\decisionTreeClass}$ can be solved in time $O(|\aFeatureSet||\aDecisionTree|^2)$.
\end{theorem}

\begin{proof}
    We describe the algorithm to solve the problem. Given the decision tree $\aDecisionTree$ and the entity $e$, compute $k = T(e)$. Now, relabel all leaves with label $k$ to 1, and the rest to 0. The obtained tree $T_B$ is now a boolean decision tree such that $T_B(e) = 1$. It can be easily shown that $\suffReasons{}(T, e) = \suffReasons{}(T_B, e)$.

    Using Lemma~\ref{lemma:suf_reasons_as_hitting_set} we can construct a hypergraph $H$ from $\dnf{T_B}^e$ such that $x$ is relevant for $T_B(e)$ if and only if $x$ belongs to a minimal hitting set of $H$. We can decide this using the algorithm from Lemma~\ref{lemma:structured_hitting_set} in time $O(|H|\deg{}(v)) = O(|\aFeatureSet||\aDecisionTree|^2)$.
\end{proof}

As explained in \cite{audemard2021explanatory}, the previous algorithm can actually be used to compute all relevant features in time $O(|\aFeatureSet||\aDecisionTree|^2)$ by removing from $H$ all hyperedges $B$ that contain another hyperedge. In the resulting hypergraph any node belonging to some hyperedge must belong to a minimal hitting set, and thus be a relevant feature.

Many features can be relevant, and as a consequence relevancy may not be a notion strong enough to rank features properly. We now propose two problems that generalize the notion of relevancy, show that are \NPcomplete{} in general, but that they can be solved efficiently for restricted instances.

\begin{proposition}\label{prop:counting_suff_reasons}
    The problem of counting, given a boolean binary decision tree $\aDecisionTree$, an entity $e$ and a feature $x$, the number of sufficient reasons for $T(e)$ containing $x$, is \sharpPcomplete. Also, deciding if there are at least $k$ sufficient reasons for $T(e)$ containing $x$ can be done in time $O(k|\aFeatureSet|^{k+2}|\aDecisionTree|)$.
\end{proposition}

\begin{proof}
    The hardness results follows by a reduction from the problem of counting the minimal vertex covers of a graph, which is \sharpPcomplete{} \cite{okamoto2005linear}[Theorem 8]. 
    
    Let $G = (V, E)$ be a graph, and consider $G' = (V \cup \{v, w\}, E \cup \{vw\})$ where $v,w \notin V$. Then, it holds that the number of minimal vertex covers of $G$ coincides exactly with the number of minimal vertex covers of $G'$ containing node $v$. Thus, using Lemma~\ref{lemma:from_hitting_to_trees} we can build a binary boolean decision tree $\aDecisionTree_{G'}$ and an entity $e$ such that the number of minimal vertex covers of $G'$ containing $v$ coincides exactly with the number of sufficient reasons of $\aDecisionTree_{G'}(e)$ containing $v$.

    Regarding the algorithm, it follows immediately using Lemmas~\ref{lemma:suf_reasons_as_hitting_set} and~\ref{lemma:construct_distinct_hitting_sets}.
\end{proof}

\begin{proposition}\label{prop:suff_reasons_with_structure}
    The problem of deciding, given a decision tree $\aDecisionTree$, an entity $\anEntity$ and a subset of features $Y \subseteq \aFeatureSet$, whether there is a sufficient reason $S \in \suffReasons{}(\aDecisionTree, e)$ such that $Y \subseteq S$, is \NPcomplete{}. Also, it can be solved in time $O(|X||T|^{|Y|+1})$.
\end{proposition}

\begin{proof}
    The \NP{} membership is immediate: given a subset of features it is possible to decide whether they are a sufficient reasons \cite{huang2023feature}. For the hardness, by Lemma~\ref{lemma:from_hitting_to_trees} we can reduce the problem of computing a minimal hitting set containing some set of nodes (which is \NPhard{} by Lemma~\ref{lemma:construct_distinct_hitting_sets}) to our problem. 
    
    For the algorithm, apply Lemmas~\ref{lemma:suf_reasons_as_hitting_set} and~\ref{lemma:structured_hitting_set}.
\end{proof}

Even though these problems turn out to be intractable in general, the algorithmic results from Propositions~\ref{prop:counting_suff_reasons} and~\ref{prop:suff_reasons_with_structure} imply that there exist efficient algorithms generalizing slightly the notion of relevancy in decision trees. Meanwhile, for models that are more expressive than $\fbddClass{}$s this is known to be impossible because of the intractability result for $\pRelevantFeatures{\fbddClass{}}$. We observe that the case for \obddClass{}s remains open.

\subsection{Necessary features}

Consider the following problem:

\defProblem{$\pNecessaryFeatures{\aModelClass}$}
    {A model $M \in \aModelClass$ over $\aFeatureSet$, an entity $e \in \entities{\aFeatureSet}$ and a feature $\aFeature \in \aFeatureSet$.}
    {Is $x$ necessary for the prediction $M(e)$?}

It was proven in \cite{audemard2021explanatory} that this problem can be solved in time $O(|\aFeatureSet||M|^2)$ for $\mathcal{M} = \booleanDecisionTreeClass$. Furthermore, in \cite{huang2023feature} it was shown that the problem is tractable for any class of models for which the problem of deciding if a subset of features is a reason (not requiring minimality) is tractable. This implies tractability of the problem for \fbddClass{}s and many other families of circuits such as \ddnnfClass{}s. We extend this result using the following characterization, which highlights the fact that the condition of being necessary is strong:

\begin{theorem}[Characterization of necessary features]\label{teo:characterize_necessary_features}
    Let $\aModelClass$ be a model over features $\aFeatureSet$ and $\anEntity \in \entities{\aFeatureSet}$.
    Then, a feature $x \in \aFeatureSet$ is necessary for $M(e)$ if and only if $M(e) \neq M(e_{x=b})$ for some $b \in \aDomain_x$.
\end{theorem}

\begin{proof}
    Observe that $x$ is necessary for $M(e)$ if and only if $x \in S$ for every $S \in SR(M, e)$, which happens if and only if $X \setminus \{x\}$ is not a reason for $M(e)$ (since otherwise there would be some sufficient reason $S \subseteq \aFeatureSet \setminus \{x\}$)\footnote{This was proven in \cite{huang2023feature}[Proposition 2].}. It holds that $X \setminus \{x\}$ is a reason if and only if $M(e) = M(e_{x=b})$ for every $b \in \aDomain_x$. 
\end{proof}

Thus, it is easy to compute necessary features for any model class that can be evaluated efficiently as long as $|\aDomain_x|$ is small.

\begin{corollary}\label{coro:necessary_complexity_depends_on_domain_x}
    The problem $\pNecessaryFeatures{\aModelClass}$ can be solved in time $O(eval(M) |\aDomain_x|)$, where $eval(M)$ denotes the time complexity required to evaluate the model $M$.

    In particular this implies that the problem is tractable for the class of boolean binary classifiers (which contains the class of \dnfClass{}s and general binary decision diagrams).
\end{corollary}

Corollary~\ref{coro:necessary_complexity_depends_on_domain_x} implies that the necessary features for a prediction can be computed efficiently for any model as long as the domain of the features is bounded.

When considering numerical features, Corollary~\ref{coro:necessary_complexity_depends_on_domain_x} is useless. Nonetheless, we observe that for ``comparison-based models'' (i.e. models that only query values of the entity through comparisons of the form $e(x) \leq b$) the non-finite (and even non-discrete) domain of numerical features can be \textit{discretized} using well-known ideas \cite{darwiche2022computation}. For example, in the case of our version of decision trees, it suffices to look for all nodes querying the value of feature $x$ as $x \leq t_1,\ldots,x\leq t_m$ with $t_1\leq t_2 \leq \ldots \leq t_m$ and then consider the finite classes induced by the intervals $[\minvalue_x, t_1], (t_1, t_2], \ldots, (t_{m-1}, t_m], (t_m, \maxvalue_x]$. Such a strategy would also work for any graph-based classifier à la \cite{huang2021efficiently}. Note that the number of classes that have to be considered grows linearly with the number of nodes, and thus we conclude tractability of the problem.

Nonetheless, for our class \decisionTreeClass{} we can compute all necessary features in linear time.

\begin{proposition}\label{prop:linear_time_necessary_trees}
    Given a decision tree $\aDecisionTree$ over features $\aFeatureSet$ and $e \in \entities{\aFeatureSet}$ it is possible to compute all the necessary features for $T(e)$ in $O(|\aFeatureSet| + |T|)$.
\end{proposition}

\begin{proof}
    We assume that $\aDecisionTree$ is a boolean decision tree and that $\aDecisionTree(e) = 1$, which can be done without loss of generality using the techniques employed for the algorithm from Theorem~\ref{teo:relevant_features_for_trees}.

    \newcommand{\primitive}[0]{\texttt{taut}}

    We will need the primitive $\primitive{}_\aDecisionTree(v, e, x, (a, b])$, which: given a node $v$ from a boolean decision tree $\aDecisionTree$, an entity $e$, a numerical feature $x$ and a range $(a, b]$, decides whether $T_v(e_{x=c}) = 1$ for all $c \in (a, b]$. It can be implemented in linear time using recursion as detailed in Figure~\ref{fig:algorithm_for_numerical_features}. We also need an analogous version $\primitive{}_\aDecisionTree(T, e, x, Z)$ where $x$ is a categorical feature and $Z \subseteq \aDomain_x$, which we detail in Algorithm~\ref{fig:algorithm_for_categorial}.

    \begin{figure}
        \centering
        \begin{algorithm}[H]

        \begin{algorithmic}[1]
        
        \REQUIRE $v$ is a node from $\aDecisionTree$, $e$ an entity for $\aDecisionTree$ and $x$ a numerical feature
        
        \IF{$(a, b] = \emptyset$}
            \RETURN 1
        \ELSIF{$v$ is a leaf with label $b$}
            \RETURN $b$
        \ENDIF
        \STATE $w_1 \gets $ left child of $v$
        \STATE $w_2 \gets $ right child of $v$
        \IF{$\feature{}(v) \neq x \wedge l(v)(e) = 1$}
            \RETURN $\primitive{}_{T}(w_1, e, x, (a, b])$
        \ELSIF{$\feature{}(v) \neq x \wedge l(v)(e) = 0$}
            \RETURN $\primitive{}_{T}(w_2, e, x, (a, b])$
        \ELSIF{$\feature{}(v) = x \wedge \val{}(v) \in (a, b)$}
            \RETURN $\primitive{}_{T}(w_1, e, x, (a, \min\{b, \val{}(v)\}]) \wedge \primitive{}_{T}(w_1, e, x, (\max\{a,\val{}(v)\}, b])$
        \ENDIF
        \end{algorithmic}
        \caption{$\primitive{}_\aDecisionTree(v, \anEntity, x, (a, b])$}
    \end{algorithm}
        \caption{Algorithm to decide whether $T_v(x_{x=c}) = 1$ for all $c \in (a, b]$}
        \label{fig:algorithm_for_numerical_features}
    \end{figure}

    \begin{figure}
        \centering
        \begin{algorithm}[H]

        \begin{algorithmic}[1]
        
        \REQUIRE $v$ is a node from $\aDecisionTree$, $e$ an entity for $\aDecisionTree$ and $x$ a categorical feature
         \IF{$Z = \emptyset$}
            \RETURN 1
        \ELSIF{$v$ is a leaf with label $b$}
            \RETURN $b$
        \ENDIF
        \STATE $w_1 \gets $ left child of $v$
        \STATE $w_2 \gets $ right child of $v$
        \IF{$\feature{}(v) \neq x \wedge l(v)(e) = 1$}
            \RETURN $\primitive{}_{T}(w_1, e, x, Z$
        \ELSIF{$\feature{}(v) \neq x \wedge l(v)(e) = 0$}
            \RETURN $\primitive{}_{T}(w_2, e, x, Z$
        \ELSIF{$\feature{}(v) = x \wedge \val{}(v) \in (a, b)$}
            \RETURN $\primitive{}_{T}(w_1, e, x, Z \cap \values{}(v)) \wedge \primitive{}_{T}(w_2, e, x, Z \cap (\aDomain_x \setminus \values{}(v)))$
        \ENDIF
        \end{algorithmic}
        \caption{$\primitive{}_\aDecisionTree(v, \anEntity, x, Z)$}
    \end{algorithm}
        \caption{Algorithm to decide whether $T_v(e_{x=b}) = 1$ for all $b \in Z$.}
        \label{fig:algorithm_for_categorial}
        
    \end{figure}

    Given these functions, it holds that a numerical feature $x$ is necessary for $T(e)$ if and only if $\primitive{}_T(r, e, x, \aDomain_x) = 1$ where $r$ is the root of $\aDecisionTree$. Thus, we can find each necessary feature in $O(|\aDecisionTree|)$. In order to achieve a linear time algorithm to compute all the necessary features we improve on this idea. 
    
    Let $P$ be the path traversed from the root to some leaf when processing entity $e$, and let $v_0,\ldots, v_k$ be the nodes of $P$ where a comparison against $x$ is made, naming $l(v_i) \equiv x \leq t_i$. We define a sequence of intervals $I_0,\ldots,I_k, I_{k+1}$ that partition $(m_x, M_x]$ as follows\footnote{We assume that the interval is open on the left side for simplicity.}:

    \begin{enumerate}
        \item If $l(v_0)(e) = 1$, let $I_0 = (t_0, M_x]$. Otherwise, let $I_0 = (m_x, t_0]$.

        \item For $1 \leq i \leq k$, if $l(v_i)(e) = 1$ then $I_i = ([m_x, M_x] \setminus \cup_{j=0}^{i-1} I_{j-1}) \cap (t_i, M_x]$, and otherwise $I_i = ([m_x, M_x] \setminus \cup_{j=0}^{i-1} I_{j-1}) \cap (m_x, t_{i}]$.

        \item $I_{k+1}$ is equal to $(m_x, M_x] \setminus \cup_{i=0}^k I_i$.
    \end{enumerate}

    Note that $I_0 \ldots I_{k+1}$ are indeed intervals of the form $(a, b]$ because $I_i = \bigcap_{j=0}^{i-1} I_j^c \cap (a_i, b_i]$ where $I_j^c = (m_x, M_x] \setminus I_j$ and $(a_i, b_i]$ is an interval that depends on the value of $l(v_i)(e)$. $I_i$ denotes the interval of values for $x$ such that any entity of the form $e_{x=c}$ for $c \in I_i$ will reach node $v_i$ when being processed by $e$ and then diverge from the path $P$.

    We claim that $\primitive{}_T(r, e, x, (m_x, M_x]) = \bigwedge_{i=0}^k \primitive{}_T(v_i, e, x, I_i)$. To prove this, it is enough to show that $\primitive{}_T(r, e, x, I_i) = \primitive{}_T(v_i, e, x, I_i)$, considering that the intervals partition the range of values for feature $x$. The equality holds because of the definition of the intervals and our previous remarks. We note that there is no need to consider the interval $I_{k+1}$ because those entities whose value for $x$ is in this range $I_{k+1}$ will follow path $P$ until the end, and thus will be evaluated to 1.

    For the categorical features an analogous argument can be made, and thus the call to $\primitive{}_T(r, e, x, \aDomain_x)$ can be decomposed as $\bigwedge_{i=0}^k \primitive{}_T(v_i, e, x, Z_k)$ where the sets $\{Z_i\}_{0 \leq i \leq k+1}$ partition $\aDomain_x$.

    With this decomposition we can now obtain all necessary features in linear time. The algorithm consists in traversing $P$, keeping for each feature the current interval (in the case of numerical features) or subset of values (for the categorical ones) consistent with all entities reaching the current node. Then, at every such node we call $\primitive{}_T$ according to the feature being considered in the node and the current set of consistent values. Note that each call to $\primitive{}_T$ will process different nodes, and thus the whole complexity is $O(|\aFeatureSet| + |T|)$.

\end{proof}

Using similar ideas one can compute all the necessary features for any \fbddClass{} model.

\begin{proposition}\label{prop:necessary_for_fbdds_linear}
    Given an \fbddClass{} $D$ over features $X$ and $e \in \entities{X}$ it is possible to compute all necessary features for $M(e)$ in $O(|\aFeatureSet| + |D|)$.
\end{proposition}

\begin{proof}
    We use the same ideas from Proposition~\ref{prop:linear_time_necessary_trees}. Let $D$ be an \fbddClass{} and $e$ some entity. Through dynamic programming we can compute, for each node $v$ in $D$, the value $D_v(e)$ in $O(|D|)$ time observing that the following relation holds:

    \begin{align*}
        D_v(e) = \begin{cases}
            b & v \text{ is a leaf with label } b\\
            D_{\texttt{left}(v)}(e) & e(\feature{}(v)) = 0\\
            D_{\texttt{right}(v)}(e) & e(\feature{}(v)) = 1 
        \end{cases}
    \end{align*}

    Let $P$ be the path followed in $D$ when processing $e$. Note that a feature $x$ is necessary if and only if $x = \feature{}(v)$ for some $v \in P$ and $D_{\texttt{left}(v)} = D_{\texttt{right}(v)} = D(e)$. To see this, note that there is a unique node $v \in P$ such that $x = \feature{}(v)$ because of the read-once property of \fbddClass{}s. Also, nodes reachable from $v$ will not use the feature $x$ again, and thus $D(e_{x=0}) = D_{\texttt{left}(v)}(e)$ and $D(e_{x=1}) = D_{\texttt{right}(v)}(e)$.
    
    Therefore, after the precomputation step we can traverse the path $P$ and decide in $O(1)$ for each feature whether it is necessary.
\end{proof}

Naturally, the previous algorithm can be used in \obddClass{}s because they are particular cases of \fbddClass{}s. We note that it cannot be straightforwardly generalized to arbitrary diagrams because we required the read-once property for the correctness of the feature necessity decision criterion.

As a closing remark, we mention that the characterization from Theorem~\ref{teo:characterize_necessary_features} is bittersweet, since it implies that the condition of being necessary is particularly strong, and thus maybe suboptimal regarding explainability tasks.

\subsection{Useful features}

Consider the following problem:

\defProblem{$\pUsefulFeatures{\aModelClass}$}
    {A model $M \in \aModelClass$ over $\aFeatureSet$ and a feature $\aFeature \in \aFeatureSet$.}
    {Is $x$ useful for $M$?}

We can characterize useful features as those that are relevant for some entity and at the same time as those that are necessary for some entity.

\begin{theorem}[Characterization of useful features]\label{prop:useful_features_are_necessary_for_someone_and_also_relevant}
    Let $\aModelClass$ be a model over $\aFeatureSet$ and $x \in X$ a feature. Then, the following are equivalent:

    \begin{itemize}
        \item[i)] $x$ is useful for $\aModelClass$.
        \item[ii)] There is some $e \in \entities{\aFeatureSet}$ such that $x$ is necessary for $M(e)$.
        \item[iii)] There is some $e \in \entities{\aFeatureSet}$ such that $x$ is relevant for $M(e)$.
    \end{itemize}
\end{theorem}

\begin{proof}
    $i)$ is equivalent to $ii)$ because of Theorem~\ref{teo:characterize_necessary_features}. $ii)$ implies $iii)$ because every necessary feature is also relevant.

    To prove that $iii)$ implies $i)$ let $e \in \entities{\aFeatureSet}$ be an entity such that $x$ is relevant, and let $S \in SR(M, e)$ be a sufficient reason such that $x \in S$. Then, since $S$ is minimal, $S \setminus \{x\}$ is not a reason, and thus there is some entity $e' \in \consistWith{e}{S \setminus\{x\}}$ such that $M(e') \neq M(e)$. Note that since $S$ is a reason it must be the case that $e'(x) \neq e(x)$. 
    
    It holds that $M(e'_{x=e(x)}) = M(e) \neq M(e')$, and therefore, we conclude that $x$ is useful.

\end{proof}

Let $\pEquiv{\aModelClass}$ be the problem of deciding, given two models $M_1,M_2 \in \mathcal{M}$, whether $M_1 \equiv M_2$. We observe that this problem is related to $\pUsefulFeatures{\aModelClass}$:

\begin{proposition}\label{prop:useful_reduces_b2b_to_equivalence}
    Let $\aModelClass$ be a class of models closed by conditioning. Then, $\overline{\pUsefulFeatures{\aModelClass}} \leq_p \pEquiv{\aModelClass}$. 

    Let $\aModelClass$ be a class of models closed by disjoint disjunction. Then, $\pEquiv{\aModelClass} \leq_p \overline{\pUsefulFeatures{\aModelClass}}$.
\end{proposition}

\begin{proof}
    For the first part, observe that $x$ is not useful for a model $M$ if and only if $M = M_{x=b}$ for any $b \in \aDomain_x$. For the second part, note that $M_1 \equiv M_2$ if and only if $x$ is not useful for the model $(M_1 \wedge x) \vee (M_2 \wedge \overline{x})$.
\end{proof}

Observe that $\pUsefulFeatures{\aModelClass}$ is in \textsc{coNP} in general because we can decide if $x$ is not useful by guessing an entity $e$ and a value $b \in \aDomain_x$. Together with Proposition~\ref{prop:useful_reduces_b2b_to_equivalence} this implies the $\coNP$-completeness of the problem $\pUsefulFeatures{\dnfClass{}}$. Similarly, it implies tractability for the classes $\decisionTreeClass{}$ and $\obddClass{}$. Finally, for classes such as $\fbddClass{}$ or $\ddnnfClass{}$ the complexity remains unknown because the complexity of both $\pEquiv{\fbddClass{}}$ and $\pEquiv{\ddnnfClass{}}$ is unknown.

If $x$ is useful then there is some entity for which it is necessary. Moreover, we can assign an ``importance score'' to a feature $x$ by considering the number of entities for which it is necessary. Note that this makes sense only for models with only categorical features, since otherwise the number of entities could be infinite.

This scoring idea can be related to the one considered in \cite{bertossi2020causality} based on the notion of \textit{counterfactual cause} \cite{halpern2005causes}. In their context, a counterfactual cause for an entity $e$ on the prediction of a boolean classifier $M$ such that $M(e) = 1$ is a feature $x$ alongside a value $b$ such that $M(e_{x=b}) = 0$. Our proposed score assigns importance to a feature proportional to the number of entities admitting a counterfactual cause based on that feature.

This score can be computed through a reduction to model counting for models closed by conditioning, negation and conjunction. From now on, given a boolean model $M$ we denote by $CT(M)$ the number of entities that $M$ accepts.

\begin{proposition}\label{prop:counting_for_classes_of_models}
    Let $M \in \aModelClass$ be a categorical boolean model whose model class is closed by conditioning, negation and conjunction. Then, the value $|\{e \in \entities{\aFeatureSet}: x \text{ is necessary for } M(e)\}|$ can be computed using 2 calls to model counting for class $\aModelClass$.     
\end{proposition}

\begin{proof}
    We will count the number of entities for which $x$ is not necessary. By Theorem~\ref{teo:characterize_necessary_features} $x$ is not necessary for $e$ if and only if $k = M(e) = M(e_{x=b})$ for all $b \in \aDomain_x$. Thus, the number of entities for which $x$ is not necessary is $CT(\bigwedge_{b \in \aDomain_x} M_{x=b}) + CT(\lnot \bigwedge_{b \in \aDomain_x} M_{x=b})$ 
\end{proof}

\begin{corollary}\label{coro:count_necessary_decision_trees}
    Given a boolean decision tree $\aDecisionTree \in \booleanDecisionTreeClass$ and a feature $x$, it is possible to compute the number of entities for which $x$ is necessary in $O(|T|^2)$.
\end{corollary}

\begin{proof}
    Given two decision trees $T_1,T_2$, the model $T_1 \wedge T_2$ can be represented by a tree of size $O(|T_1||T_2|)$ appending to each leaf with label 1 of $T_1$ a copy of $T_2$. Moreover, model counting in trees can be done in linear time \cite{darwiche2002knowledge}. Thus, in this case the algorithm from Proposition~\ref{prop:counting_for_classes_of_models} can be implemented in time $O(|T|^2)$.
\end{proof}

We remark that Proposition~\ref{prop:counting_for_classes_of_models} can be generalized for $k$-class models as long as it is possible to ``booleanize'' the $k$-class model $M$ by identifying one class $c$ as 1 and the others as 0. If the classes of $M$ are $\{0,\ldots, k-1\}$ and $M^c$ for $0 \leq c < k$ denotes the booleanized version of $M$ where $c$ is identified with $1$, then the number of entities for which $x$ is not necessary can be computed as $\sum_{i=0}^{k-1} CT(\bigwedge_{b \in \aDomain_x} M^c_{x=b})$.

\section{Experiments}\label{sec:experiments}

In this section we experiment with three different datasets: the California Housing Dataset, the Bike Sharing Demand Dataset and the Adult Income Dataset\footnote{All of them are available in Kaggle.}. We aim to understand whether the ranking induced by our scoring scheme is consistent with the importance ranking of the features of these well-known datasets.

To test this hypothesis we train different models for each dataset and compute the scores for all the features. We will train decision trees in order to use the efficient algorithm from Corollary~\ref{coro:count_necessary_decision_trees}. The ground truth for each dataset (i.e. the true feature importance ranking) is obtained by analyzing the different reports from the Kaggle users as well as some related papers \cite{yu2023multivariate,xu2024research,gold2024statistical}. These datasets were chosen precisely because of the abundance of studies involving them, as well as the fact that they have between 10k and 50k entries. We will also compare the ranking induced by our score with the one induced by the SHAP-score using the \texttt{shap} python library. 

All the experiments are available in our repository\footnote{\url{https://github.com/Andial66/FeatureUsefulness}}.

\paragraph{Preparing the datasets}

Our scoring scheme is well-defined when all features are categorical, and thus we pre-process all datasets to ensure this. We will consider using 3, 4, 5, and 6 bins for each numerical feature, and to discretize them we will use the \texttt{KBinsDiscretizer} from \texttt{sklearn} with the \texttt{uniform} strategy. We detail the process for each dataset:

\begin{enumerate}
    \item \textbf{California Housing Dataset}: all features are numerical, and thus we discretize all of them.
    \item \textbf{Bike Sharing Dataset}: The features \texttt{season}, \texttt{yr}, \texttt{holiday}, \texttt{workingday} and \texttt{weathersit} are already categorical and thus are not modified. The rest of the features are discretized. We note that some of these remaining features are also categorical (such as \texttt{hr}) but we still discretize them using smaller bins to ensure that the algorithm from Corollary~\ref{coro:count_necessary_decision_trees} computes the score in a matter of seconds\footnote{All the experimental results were obtained in less than 10 minutes using a standard laptop.}.
    \item \textbf{Adult Income Dataset}: The features \texttt{race} and \texttt{sex} are categorical and we do not modify them. The other ones are discretized using the chosen number of bins, even if they are already categorical.
\end{enumerate}

The train-test split is made taking 20\% of the dataset as test.

\begin{figure}
    \centering
    \includegraphics[width=\linewidth]{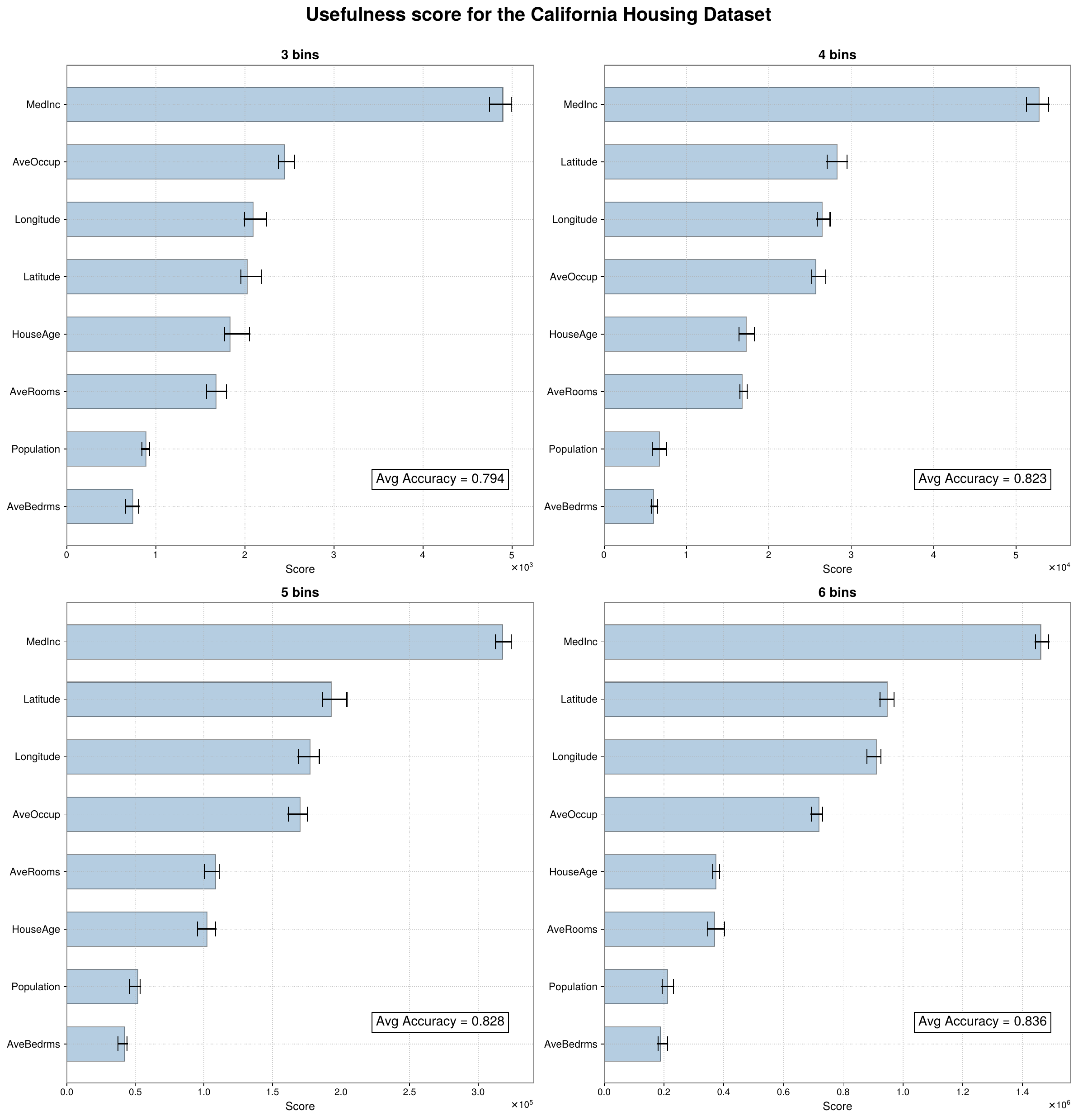}
    \caption{Results for the California Housing Dataset. For each number of bins we train 20 models and show the average score of each feature across all of them, alongside the Q1 and Q3 quartiles. We also display the average accuracy of the models.}
    \label{fig:california_dataset_results}
    
\end{figure}

\paragraph{Ground truth}

We state the ground truth considered for each dataset:

\begin{enumerate}
    \item \textbf{California Housing Dataset}: the most important feature from this dataset is \texttt{MedInc} (Median Income), while the location, deduced from the \texttt{longitude} and \texttt{latitude}, is often understood to be the second most relevant one. \texttt{HouseAge} is usually understood to be the third most influential feature. Among the least relevant ones we have \texttt{Population} and \texttt{AveBedrms}. 
    \item \textbf{Bike Sharing Dataset}: the most important feature is \texttt{hr} (hour of the day), while other influential ones are \texttt{temp} (temperature) and \texttt{hum} (humidity). Features like \texttt{weekday} or \texttt{holiday} are considered to be the least relevant.
    \item \textbf{Adult Income Dataset}: features like \texttt{education-num}, \texttt{capital-gain} and \texttt{relationship} are good predictors. On the other hand, \texttt{fnlwgt} (final weight), \texttt{race} and \texttt{education} are not as relevant.
\end{enumerate}

\paragraph{The models}

We train decision trees for each dataset using the \texttt{DecisionTreeClassifier} from the \texttt{sklearn} library. More precisely, for each bin size we train 20 trees and compute the average score across the different models. To regularize the trees and avoid overfitting we fix the number of leaves as $100 \times \#bins$ for the California Housing Dataset and Bike Sharing Dataset, and $150 \times \#bins$ for the Adult Income Dataset (the difference is due to the fact that the Adult Income Dataset has $\sim$30k entries while the other ones have less than 20k).

\begin{figure}
    \centering
    \includegraphics[width=\linewidth]{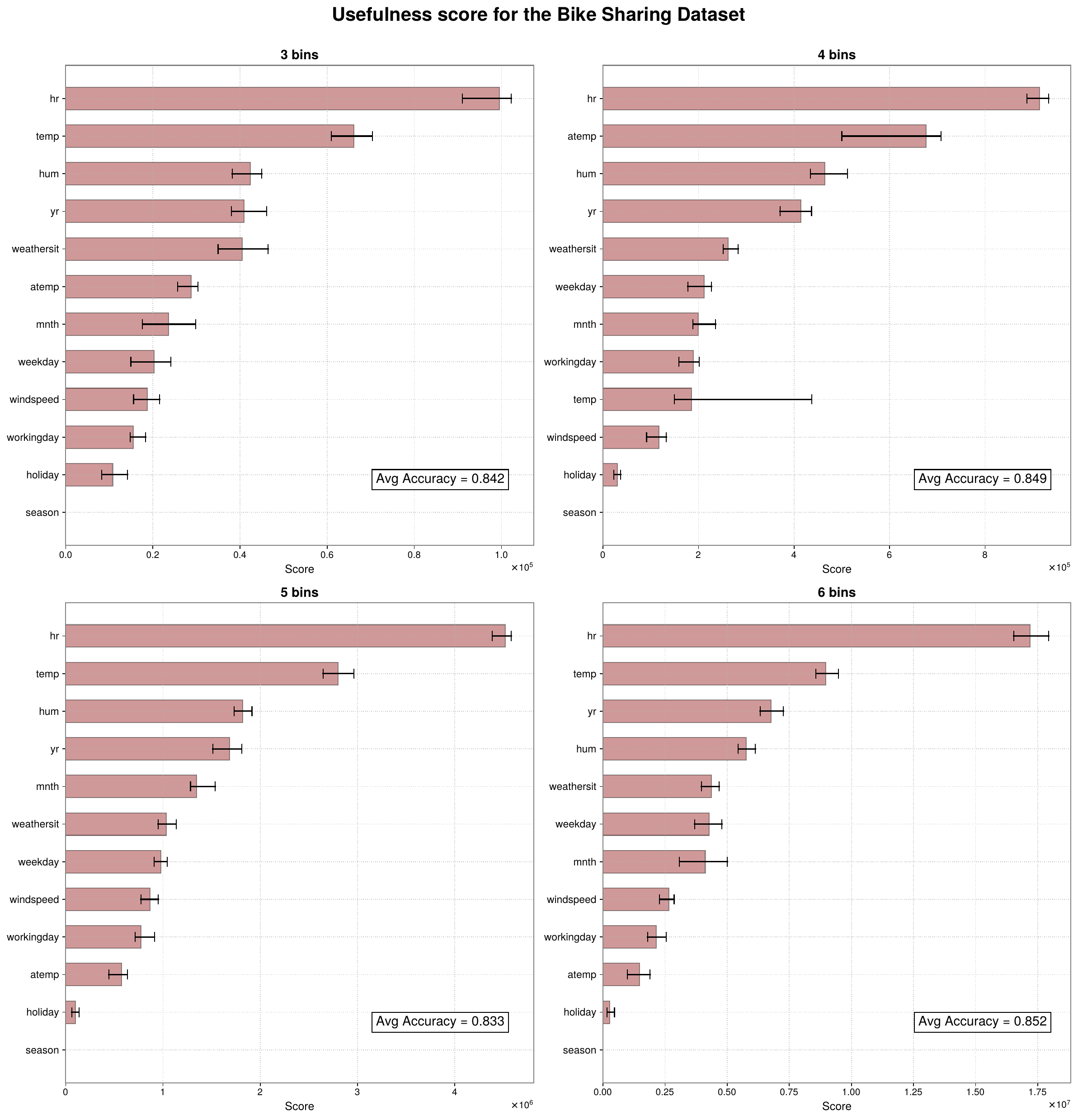}
    \caption{Results for the Bike Sharing Dataset. The displayed plots are analogous to the ones from Figure~\ref{fig:california_dataset_results} but for this dataset.}
    \label{fig:bike_sharing_dataset_results}
    
\end{figure}

\paragraph{The SHAP-scores} To compute the SHAP-score for each feature we compute the average of the SHAP values for each entity in the training set, as is proposed by the \texttt{shap} framework. The experimentation involving the comparison between these scores and the usefulness score only considers the datasets binarized with 6 bins, and builds 20 models for each dataset. For each model we compute both rankings, and check the size of the intersection of the top-1, top-3, top-5 and top-7 features.

\paragraph{Results}

Our results can be seen in Figures~\ref{fig:california_dataset_results}, \ref{fig:bike_sharing_dataset_results} and \ref{fig:adult_income_dataset_results}. We observe that the values in the x-axis have different values for each number of bins because depending of the number of bins there are more or less entities.

For the case of the California Housing Dataset we see that the ranking induced by the usefulness score coincides in most cases with the ground truth, specially when considering the most and less influential feature. We note that for the case of 3 bins the feature \texttt{AveOccup} is ranked second, contradicting what we expected. We believe this is an issue raised by the categorization of the \texttt{longitude} and \texttt{latitude} variables: when they are discretized into three bins in an uniform manner the feature is not granular enough to distinguish whether the location corresponds to a wealthy neighborhood. Thus, when we increase the number of bins these features gain importance.

\begin{figure}
    \centering
    \includegraphics[width=\linewidth]{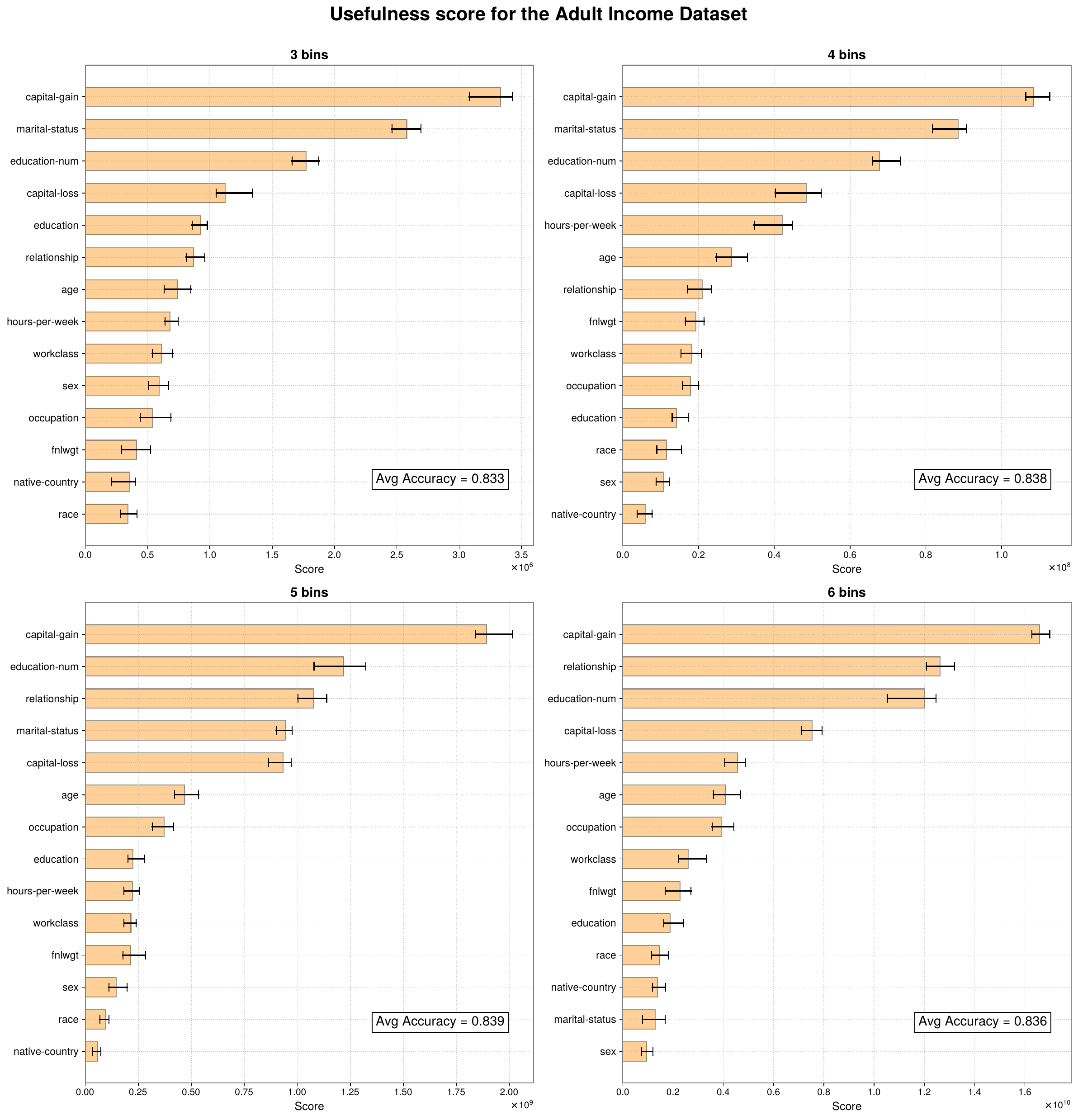}
    \caption{Results for the Adult Income Dataset. The displayed plots are analogous to the ones from Figure~\ref{fig:california_dataset_results} but for this dataset}
    \label{fig:adult_income_dataset_results}
    
\end{figure}

Regarding the Bike Sharing Dataset we note that again our score correlates strongly with the ground truth for all numbers of bins, specially when looking at the most relevant feature. We note that the feature \texttt{atemp} (perceived temperature) that was ranked second for the case of 4 bins actually encodes information correlated with the \texttt{temp} variable that we expected to be better ranked. Moreover, the scores for these two features have a high variance (as seen by the quantiles) and thus it is likely that some models learn to use \texttt{atemp} while others use \texttt{temp}, but in both cases our score assigns importance to the information related to the temperature in general.

As for the Adult Income Dataset we again see that the induced ranking correlates with the ground truth when considering the most informative feature. Observe that the three features \texttt{capital-gain}, \texttt{education-num} and \texttt{relationship} are ranked top 3 in the models with 5 and 6 bins. Meanwhile, in the case of 3 and 4 bins we note that \texttt{marital-status} seems to have taken the place from \texttt{relationship}. We believe that this unexpected behavior is explained using the previous arguments: when few bins are used, the feature \texttt{relationship} might not be as informative; and moreover the feature that took its place is probably correlated with it. Note that among the least informative features we have \texttt{race} and \texttt{fnlwgt} as expected.

Finally, in Table~\ref{tab:shap_vs_useful} we see the comparison between the rankings induced by the usefulness score and the shap-scores. In average they do not differ a lot, specially when considering the simpler datasets (California and Bike Sharing). We note that for the Adult Income Dataset the rankings never match the top-feature: shap always prioritizes \texttt{relationship}, while the usefulness score prioritizes \texttt{capital-gain}, Nonetheless, they usually coincide in the top-3 ranking. Also, notice that based on our literature review \texttt{capital-gain} is the most important feature, and thus the usefulness score might be the one picking the best feature instead of shap. 

\begin{table}[h!]
\centering
\caption{Average intersection between the ranking induced by the usefulness score and the SHAP-scores.}
\label{tab:shap_vs_useful}
\begin{tabular}{lcccc}
    \toprule
     & top-1 & top-3 & top-5 & top-7 \\
    \midrule
    California   & 1    & 3    & 4.7  & 6.7  \\
    Bike Sharing & 1    & 2.7  & 3.9  & 5.2  \\
    Adult Income & 0    & 2.85 & 3.7  & 6.25 \\
    \bottomrule
\end{tabular}
\end{table}

We remark that the computation of the usefulness score is extremely fast using the simple algorithm from Corollary~\ref{coro:count_necessary_decision_trees}, and after our experiments the ranking induced by it seems qualitatively competitive with the SHAP-scores, while also capturing intuitive knowledge from the three datasets we tested. We believe this provides evidence for the practical utility of the scoring scheme.

\section{Conclusions}\label{sec:conclusions}

In this work we extended previous results regarding the computation of relevant and necessary features, and also proposed new generalizations for the notion of relevancy. Moreover, we defined a global notion of feature importance, and showed that it is related to the local ones induced by relevancy and necessity. We studied many problems involving these parameters, proving them to be intractable for complex classes of models in some cases (such as detecting \textit{usefulness} in \dnfClass{}s) but also finding tractable restrictions (for example, detecting \textit{necessity} in \dnfClass{}s). For the tractable cases we described our algorithms with precision, allowing for an easy implementation. In addition, some of them run in linear time, and thus are highly efficient.

Regarding feature relevancy, we extended the results from \cite{audemard2021explanatory} and \cite{huang2022tractable} considering more general decision trees, similar to those studied in \cite{darwiche2022computation}. We also proposed two generalizations of the problem of computing relevant features, which we showed to be intractable in general (Propositions~\ref{prop:counting_suff_reasons} and~\ref{prop:suff_reasons_with_structure}) but solvable in the case of decision trees under certain hypotheses. Moreover, to attain the hardness results we proved hardness for a common vertex cover problem problem which, as far as we know, has not been studied from the point of view of computational complexity before (Lemma~\ref{lemma:construct_distinct_hitting_sets}). We note that the problem of computing relevant features for slightly more expressive models --such as \obddClass{}s -- is still open, since \cite{huang2021efficiently} only showed hardness for \fbddClass{}s.

On the topic of feature necessity, we extended the results of \cite{huang2023feature}, showing that computing necessary features is feasible for all reasonable models (i.e. those allowing evaluation in polynomial time) with categorical features (Corollary~\ref{coro:necessary_complexity_depends_on_domain_x}). Furthermore, we provided linear time algorithms for the case of decision trees with categorical and numerical features, as well as for \fbddClass{}s (Propositions~\ref{prop:linear_time_necessary_trees} and~\ref{prop:necessary_for_fbdds_linear}). We left open the case of extending this linear time algorithms for arbitrary diagrams, or rather to consider restricted families of circuits such as \ddnnfClass{}s.

Finally, we introduced the notion of feature \textit{usefulness}, and showed how it is related to relevancy and necessity (Theorem~\ref{prop:useful_features_are_necessary_for_someone_and_also_relevant}). We showed that the problem of computing useful features is related to the problem of detecting whether two models are equivalent (Proposition~\ref{prop:useful_reduces_b2b_to_equivalence}), and in particular this implied intractability for expressive models such as \dnfClass{}s, and tractability for simple ones such as decision trees and \obddClass{}s. The case for \fbddClass{}s and \ddnnfClass{}s remains open. We also considered measuring usefulness as a counting problem, and provided general algorithms that were efficient in the particular case of decision trees (Proposition~\ref{prop:counting_for_classes_of_models} and Corollary~\ref{coro:count_necessary_decision_trees}). Lastly we performed experiments for our usefulness-based scoring system using three well-known publicly available datasets, where we verified that our scores fairly correspond with both the ground truth and other popular scoring systems such as the SHAP-score.

\bibliographystyle{plain}
\bibliography{biblio}

\end{document}